\documentclass[10pt,twocolumn,twoside]{IEEEtran}

\ifCLASSOPTIONcompsoc
\else
\fi
\ifCLASSINFOpdf
\else
\fi
\usepackage{times}
\usepackage{epsfig}
\usepackage{graphicx}
\usepackage{subfigure}
\usepackage{booktabs}
\usepackage{amsthm}
\usepackage{amsmath}
\usepackage{amssymb}
\usepackage{array}
\usepackage{setspace}
\usepackage{txfonts}

\usepackage[font=footnotesize,labelfont=bf]{caption}

\usepackage{algorithm}
\usepackage{algorithmic}

\newtheorem{theorem}{Theorem}

\begin{document}
%
\title{Stochastic Feature Mapping for PAC-Bayes Classification}
%
%
%
%

\author{Xiong Li,
        and Tai Sing Lee,
        and Yuncai~Liu,~\IEEEmembership{Senior Member,~IEEE,} 
\IEEEcompsocitemizethanks{\IEEEcompsocthanksitem Xiong Li and Yuncai Liu are with the Department of Automation, Shanghai Jiao Tong University, Shanghai, China; Tai Sing Lee and Xiong Li are with the Computer Science Department, Carnegie Mellon University, PA, USA. \protect\\
E-mail: \{lixiong, whomliu\}@sjtu.edu.cn, tai@cs.cmu.edu

}}

%
%

\markboth{}{}
\IEEEcompsoctitleabstractindextext{%
\begin{abstract}
Probabilistic generative modeling of data distributions can potentially exploit hidden information which is useful for discriminative classification. This observation has motivated the development of approaches that couple generative and discriminative models for classification. In this paper, we propose a new approach to couple generative and discriminative models in an unified framework based on PAC-Bayes risk theory. We first derive the {\em model-parameter-independent stochastic feature mapping} from a practical MAP classifier operating on generative models. Then we construct a linear stochastic classifier equipped with the feature mapping, and derive the {\em explicit PAC-Bayes risk bounds} for such classifier for both supervised and semi-supervised learning. Minimizing the risk bound, using an EM-like iterative procedure, results in a new posterior over hidden variables (E-step) and the update rules of model parameters (M-step). The derivation of the posterior is always feasible due to {\em  the way of equipping feature mapping} and {\em the explicit form of bounding risk}. The derived posterior allows the tuning of generative models and subsequently the feature mappings for better classification. {\em The derived update rules of the model parameters are same to those of the uncoupled models} as the feature mapping is model-parameter-independent. Our experiments show that the coupling between data modeling generative model and the discriminative classifier via a stochastic feature mapping in this framework leads to a general classification tool with state-of-the-art performance.
\end{abstract}

}

\maketitle

\IEEEdisplaynotcompsoctitleabstractindextext

%
\IEEEpeerreviewmaketitle

\section{Introduction}

Discriminative models designed to find decision boundaries among different classes are state-of-the-art tools for classification, while probabilistic generative models seeking to model data distributions are adept in exploiting hidden information, in dealing with structured data (e.g. protein sequence with variable length) and in solving nonlinear classification problems using maximum a posterior (MAP) classifier. The complementarities of the two paradigms have been investigated~\cite{jaakkola1999maxent,Ng02ondiscriminative}, resulting in several promising works~\cite{jaakkola1999exploiting,raina2004gendis,mccallum2006gendis,perina2009fess}. The above observations have emerged from these works in the context of classification: (1) generative models provide feature mappings that simultaneously exploit hidden information, and transform structured data into a fixed dimensional feature; (2) discriminative models find an optimum decision boundaries in such a feature space under specified criterion.

Generative score space methods~\cite{jaakkola1999exploiting,perina2009fess,li2011posdiv} are
motivated by the above observations. These methods derive feature mappings from the log likelihood (or its lower bound) of generative models. These feature mappings are measures over models $P(\mathbf{x},\mathbf{h},\theta)$, taking the form of $\mathrm{E}_{P(\mathbf{h}|\mathbf{x})} [\phi(\mathbf{x},\mathbf{h},\theta)]$ where $\phi$ is a function over the observed variable $\mathbf{x}$ and the hidden variable set $\mathbf{h}$. They map observed and hidden variables into a vector of score, which are then used as features by classifiers. These methods exploit the superior abilities of generative models in exploiting hidden information and dealing with structured data. However, in these methods, generative models are isolated from the classification process and there is no principled way to tune the generative models as well as the feature mapping to improve classification. It is desirable to develop a mechanism that can couple the classifier to the generative models to allow fine-tuning of the feature mapping.

Maximum entropy discrimination~\cite{jaakkola1999maxent} provides yet another framework to exploit generative models for classification under the large margin principle. This framework, however, requires deliberately choosing conjugate priors for parameters of the generative models, which limits its application to complex models. In addition, the VC risk bound~\cite{vapnik2000nature} utilized by this method is generally loose in comparison with the PAC-Bayes bounds~\cite{mcallester1999somepacl,langford2006tutorial,germain2009pac}. Also, there are some other efforts~\cite{raina2004gendis,mccallum2006gendis} made to couple generative and discriminative models for classification. However, these methods provide no explicit feature mapping which is useful in real applications. Further, they requires re-formulating the update rules of the parameters of generative models, which is typically complex.

This paper proposes an approach based on the PAC-Bayes theory~\cite{mcallester1999somepacl,langford2006tutorial,germain2009pac} to integrate the complementary strengths of generative and discriminative models. Using the linear form of a practical MAP classifier operating on generative models, we derive the {\em model-parameter-independent stochastic feature mapping}. By the feature mapping, we meant that the feature used for classifications is a function of the input data and the hidden variables of the generative models. This is distinct from the current methods~\cite{jaakkola1999exploiting,perina2009fess,li2011posdiv} which map a data point to a feature deterministically. Then we construct a {\em stochastic classifier}, a Gibbs classifier operating on the derived feature mapping, and derive explicit PAC-Bayes risk bounds for such a classifier. By minimizing the risk bound using an EM-like iterative procedure, we derive {\em the posterior over the hidden variables} (E-step) and {\em the update rules of model parameters} (M-step). The derivation is always feasible due to the way of equipping feature mapping and the form of bounding risk. The posterior provides a bridge that allows the classifier to tune the generative models and subsequently the feature mapping for classification. The update rules of model parameters are quite simple -- essentially same to those of the uncoupled models as the feature mapping is model-parameter-independent.

\section{From MAP classifier to Stochastic Feature Mapping} \label{sec:map_sfm}

In this section, for exponential family generative models, we drive the linear form (Eq.~(\ref{eqn:disfun})) of the MAP classifier (Eq.~(\ref{eqn:L2})) based on the variational approximation (Eq.~(\ref{eqn:varinf})), and use that to derive a stochastic feature mapping. The derived feature mapping is functioning similar to~\cite{jaakkola1999exploiting,perina2009fess,li2011posdiv}. Consider the binary classification problem that assigns labels $y\in \{-1,+1\}$ to examples $\mathbf{x}\in \mathbb{R}^d$. Let $P(\mathbf{x}\,|\,\theta_y)$ be the class-conditional distributions over $\mathbf{x}$; $P(y)$ be the prior of labels. The decision rule of the MAP classifier is $\hat{y}=\max_y P(\mathbf{x}\,|\,\theta_y)P(y)$, which is equivalent to $\hat{y}=\mathrm{sign} (L(\mathbf{x};\Theta))$ where $\mathrm{sign}(a)=+1$ if $a>0$ and $\mathrm{sign}(a)=-1$ otherwise, and the discriminant function:
\begin{equation} \label{eqn:L1}
L(\mathbf{x},\Theta) =  \log {P(\mathbf{x}\,|\,\theta_{+})}- \log {P(\mathbf{x}\,|\,\theta_{-})}+b
\end{equation}
where subscripts $+,-$ are the shorts of $+1,-1$; $\Theta \!=\! \{\theta_{-},\theta_{+},b\}$; $b \!=\! \log P(y \!=\! +1) -\log P(y \!=\! -1)$. When $P(\mathbf{x}\,|\,\theta_y)$ is modeled by a generative model $P(\mathbf{x},\mathbf{h}\,|\,\theta_y)$ with a set of random hidden variables $\mathbf{h}$, it is difficult to obtain a close form of $P(\mathbf{x}\,|\,\theta_y)$ since $\int\! P(\mathbf{x},\mathbf{h}\,|\,\theta_y) {d}\mathbf{h}$ is usually intractable. We can resort to the following variational lower bound~\cite{neal1998view,jordan1999variational}:
\begin{equation} \label{eqn:varinf}
\log P(\mathbf{x}\,|\,\theta_y)\geq \mathrm{E}_{Q(\mathbf{h})}[ \log P(\mathbf{x},\mathbf{h})-\log Q(\mathbf{h})] \triangleq F(\mathbf{x},\theta_y)
\end{equation}
where $Q(\mathbf{h})$ is the variational approximate posterior of $P(\mathbf{h}\,|\,\mathbf{x})$. Then, instead of the intractable discriminant function (Eq.~(\ref{eqn:L1})), we resort to the following tractable one~\cite{jaakkola1999maxent,perina2009fess}
\begin{equation} \label{eqn:L2}
\hat{L}(\mathbf{x},\Theta) = F(\mathbf{x},\theta_+) - F(\mathbf{x},\theta_-) + b
\end{equation}

We assume the generative model $P(\mathbf{x},\mathbf{h}\,|\,\theta)$ belong to the exponential family which covers most models. We have the general form $P(\mathbf{x},\mathbf{h})=\exp\{ a(\theta)^T T(\mathbf{x},\mathbf{h}) + S(\mathbf{x},\mathbf{h})+d(\theta) \}$ where $\theta$ is the vector of parameters; $T(\mathbf{x},\mathbf{h})$ is the vector of sufficient statistics; $S(\mathbf{x},\mathbf{h})$ and $d(\theta)$ are scalar functions. Similarly, the prior over $\mathbf{h}$ is $P(\mathbf{h})=\exp \{c(\theta_h)^TT(\mathbf{h})+S(\mathbf{h})+f(\theta_h)\}$. Further, we assume that the approximate posterior of $\mathbf{h}$, for the example $\mathbf{x}$, takes the same from with its prior $P(\mathbf{h})$ but with different parameter~\cite{jordan1999variational}  $Q(\mathbf{h})=\exp \{c(\theta_h')^TT(\mathbf{h})+S(\mathbf{h})+f(\theta_h')\}$. Substituting the above formulas of $P(\mathbf{x},\mathbf{h})$ and $Q(\mathbf{h})$ into Eq.~(\ref{eqn:varinf}), it can be verified that $F(\mathbf{x},\theta)\!=\!\mathrm{E}_{Q(\mathbf{h})}[\log P(\mathbf{x},\mathbf{h})-\log Q(\mathbf{h})]\!=\! \alpha^T \mathrm{E}_{Q(\mathbf{h})}[\tilde{T}(\mathbf{x},\mathbf{h})] + \beta$, where $\alpha \!=\! (a(\theta)^T,1,-\mathbf{1}^T,-1,-1)$; $\beta\!=\!d(\theta)$; $\tilde{T}(\mathbf{x},\mathbf{h})\!=\!(T(\mathbf{x},\mathbf{h})^T,S(\mathbf{x},\mathbf{h}),(\mathrm{diag}(c(\theta_h')T(\mathbf{h})))^T,S(\mathbf{h}),f(\theta_h'))^T$. For a pair of models $\theta_{+}$ and $\theta_{-}$, Eq.~(\ref{eqn:L2}) can be written as:
\begin{eqnarray}
\hat{L}(\mathbf{x},\Theta) &\!\!\!\!\!=\!\!\!\!\!& \alpha_{+}^T \mathrm{E}_{Q(\mathbf{h}_+)} [\tilde{T}(\mathbf{x},\mathbf{h}_+)] \!-\! \alpha_{-}^T \mathrm{E}_{Q(\mathbf{h}_-)} [\tilde{T}(\mathbf{x},\mathbf{h}_-)] \!+ \beta_+ \!- \beta_- \!+ b \nonumber
\\
&\!\!\!\!\!=\!\!\!\!\!& \tilde{\alpha}^T \mathrm{E}_{Q(\mathbf{h}_+,\mathbf{h}_-)} [\phi(\mathbf{x},\mathbf{h}_+,\mathbf{h}_-) ]  \label{eqn:disfun}
\end{eqnarray}
where $\tilde{\alpha} = (\alpha_{+}^T,-\alpha_{-}^T,\beta_+-\beta_-+b)^T\!$, $\phi(\mathbf{x},\mathbf{h}_+,\mathbf{h}_-) \triangleq (\tilde{T}(\mathbf{x},\mathbf{h}_+)^T, \\\!\tilde{T}(\mathbf{x},\mathbf{h}_-)^T,\!1)^T $. Eq.~(\ref{eqn:disfun}) takes the form of the linear classifier, where $\mathrm{E}_{Q} [\phi(\mathbf{x},\mathbf{h}_+,\mathbf{h}_-)$ is considered to furnish a feature mapping. From another perspective, $\phi(\mathbf{x},\mathbf{h}_+,\mathbf{h}_-)$ can be considered as a {\em stochastic feature mapping} because the hidden variables $\mathbf{h}_+$, $\mathbf{h}_-$ are all conditioned on the example $\mathbf{x}$ and thus its value can serve as feature for identifying $\mathbf{x}$. It is considered to define a {\em stochastic feature space} because it is evaluated based on stochastic examples drawn from the posterior of $\mathbf{h}$.

\section{PAC-Bayes bound stochastic classifier}

The derived stochastic feature mapping $\phi(\mathbf{x},\mathbf{h}_+,\mathbf{h}_-)$ makes it possible to jointly learn generative models (subsequently feature mappings) and classifier. We construct a linear Gibbs classifier over this stochastic feature mapping:
\begin{equation} \label{eqn:classifer}
G_Q = \mathrm{sign}[\mathbf{w} \cdot \phi(\mathbf{x},\mathbf{h}_+,\mathbf{h}_-)] \triangleq f_{\mathbf{w}}(\mathbf{x},\mathbf{h}_+,\mathbf{h}_-)
\end{equation}
where $\mathbf{w}$ is the weight of classifier; $\mathbf{h}_+$, $\mathbf{h}_-$, $\mathbf{w}$ follow some distribution $Q$ which will be specified in Section~\ref{ssec:objfun}. Gibbs classifier with such a feature mapping offers {\em several advantages}. {\em First}, this classifier allows PAC-Bayes risk bounds that have explicit solutions for $Q(\mathbf{h}_+,\mathbf{h}_-)$ which can help tune the feature mapping for better classification; {\em Second}, the PAC-Bayes risk bound for such a classifier can be tighter than VC bounds~\cite{vapnik2000nature}; {\em Third}, the feature mapping is independent with model parameters $\theta$, making the solution of $\theta$ very simple.

\subsection{PAC-Bayes bounds for stochastic feature mapping}

Let $\mathcal{X}$ be the input space consisting of an arbitrary subset of $\mathbb{R}^d$ and $\mathcal{Y}=\{-1,+1\}$ be the output space. An example is an input-output pair $(\mathbf{x},y)$ where $\mathbf{x}\in \mathcal{X}$ and $y\in \mathcal{Y}$. In a PAC-Bayes setting~\cite{mcallester1999somepacl}, each example $(\mathbf{x},y)$ is drawn from a fixed, but unknown, probability distribution $D$ on $\mathcal{X}\times \mathcal{Y}$. Let $f(\mathbf{x},\mathbf{v}): \mathcal{X}\rightarrow \mathcal{Y}$ be any classifier with a set of variables $\mathbf{v}$. The learning task is to choose a posterior distribution $Q$ over a space $\mathcal{F}$ of classifiers and a space $\mathcal{V}$ of variables such that the Q-weight majority classifier $B_Q=\mathrm{sign}[\mathrm{E}_{(f,\mathbf{v})\sim Q} f(\mathbf{x},\mathbf{v})]$ will have the smallest possible risk on the training example set $S=\{(\mathbf{x}_1,y_1),\cdots,(\mathbf{x}_m,y_m)\}$. The output of $B_Q$ is closely related to the output of the Gibbs classifier $G_Q$ which first chooses a classifier $f$ and a vector $\mathbf{v}$ according to $Q$, and then classifies an example $\mathbf{x}$. The {\it true risk} $R(G_Q)$ and the {\it empirical risk} $R_S(G_Q)$ of this Gibbs classifier are given by:
\begin{eqnarray}
R(G_Q)&\!\!\!=\!\!\!&\mathop{\mathrm{E}}\limits_{(f,\mathbf{v})\sim Q} \mathop{\mathrm{E}}\limits_{(\mathbf{x},y)\sim D} \mathrm{I}(f(\mathbf{x},\mathbf{v}) \neq y)  \label{eqn:RGQ} \\
R_S(G_Q)&\!\!\!=\!\!\!&\mathop{\mathrm{E}}\limits_{(f,\mathbf{v})\sim Q} \frac{1}{m} \sum_{i=1}^m \mathrm{I}(f(\mathbf{x}_i,\mathbf{v}) \neq y_i) \label{eqn:RSGQ}
\end{eqnarray}
This setting is naturally accommodated by PAC-Bayes theory since $\mathbf{v}$ can be considered as a part of $f$. Among several PAC-Bayes bounds~\cite{mcallester1999somepacl,langford2006tutorial,germain2009pac,lacasse2006majvote}, the bound derived in~\cite{germain2009pac} is quite tight and gives an explicit bound for the true risk $R(G_Q)$, which allows the derivation of the posterior $Q$, in contrast to most of the other implicit bounds over $\mathrm{KL}(Rs(G_Q)\!\parallel\!R(G_Q))$.

\begin{theorem}\label{th:bound_sup}
For any distribution $D$ over $\mathcal{X}\times \mathcal{Y}$, any space $\mathcal{F}$ of classifiers, any space $\mathcal{V}$ of random variables,  any distribution $P$ over $\mathcal{F} \!\times\! \mathcal{V}$, any $\delta \in (0,1]$, and any real number $C>0$, $\forall\; Q$ over $\mathcal{F} \!\times\! \mathcal{V}$, we have:
\begin{eqnarray*}
&&\!\!\!\!\!\!\!\!\!\!\!\!\!\!\!\! \Pr \Bigg(R(G_Q) \leq \frac{1}{1-e^{-C}} \Bigg[1-\exp \Bigg\{ - C R_S(G_Q)  \\
&& \quad\quad\quad\quad\quad\quad\quad - \frac{1}{m}  \left[ \mathrm{KL}(Q \!\parallel\! P)-\ln \delta \right]  \Bigg\} \Bigg] \,\Bigg) \geq 1-\delta
\end{eqnarray*}
\end{theorem}
\noindent This is a slight extension of Corollary 2.2 of~\cite{germain2009pac}, and can be proved by replacing $f$ with $(f,\mathbf{v})$ and reapplying its proof~\cite{germain2009pac}.

The above risk bound is derived for labeled data. Here we have extended the bound to accommodate both labeled and unlabeled data for semi-supervised learning in the following theorem. The semi-supervised bound is different with~\cite{lacasse2006majvote} whose bound is implicit and has no explicit solution of $Q$.

\begin{theorem}\label{th:bound_semi}
For any distribution $D$ over $\mathcal{X}\times \mathcal{Y}$, any space $\mathcal{F}$ of classifiers, any space $\mathcal{V}$ of random variables, distribution $P$ over $\mathcal{F} \!\times\! \mathcal{V}$, any $\delta \in (0,1]$, and any real number $C>0$, $\forall\; Q$ over $\mathcal{F} \!\times\! \mathcal{V}$, we have:
\begin{eqnarray*}
&&\!\!\!\!\!\!\!\!\!\!\!\!\!\!\!\!  \Pr \Bigg ( R(G_Q) \leq \frac{1}{1-e^{-C}} \Bigg[1 - \exp \Bigg\{ -C \bigg[e_S(G_Q) + \frac{1}{2}d_S(G_Q) \bigg]  \\
&& \quad\quad\quad\quad\quad\quad\quad - \frac{1}{m} \left[\mathrm{KL}(Q \!\parallel\! P)-\ln \delta \right]  \Bigg\} \Bigg] \,\Bigg) \geq 1-\delta
\end{eqnarray*}
where the risks for labeled and unlabeled data are $e_S(G_Q)\!=\!\mathrm{E}_{(f_1,\mathbf{v}_1)\sim Q}\mathrm{E}_{(f_2,\mathbf{v}_2)\sim Q}\frac{1}{m}\sum_{i=1}^m \mathrm{I}(f_1(\mathbf{x}_i,\mathbf{v}_1)\neq y_i)) \mathrm{I}(f_2(\mathbf{x}_i,\mathbf{v}_2) \!\neq\! y_i))$ and $d_S(G_Q)=\mathrm{E}_{(f_1,\mathbf{v}_1)\sim Q}\mathrm{E}_{(f_2,\mathbf{v}_2)\sim Q}\frac{1}{m}\sum_{i=1}^m \mathrm{I}(f_1(\mathbf{x}_i,\mathbf{v}_1) \neq f_2(\mathbf{x}_i,\mathbf{v}_2))$.
\end{theorem}
\begin{proof}
Let $\mathrm{E}_f$ be the abbreviation of $\mathrm{E}_{(f,\mathbf{v})\sim Q}$. Note that $\mathrm{E}_{f_1,f_2}\mathrm{I}(f_1\neq f_2)=\mathrm{E}_{f_1,f_2} 2\mathrm{I}(f_1 \neq y)\mathrm{I}(f_2 = y) =\mathrm{E}_{f_1,f_2} 2\mathrm{I}(f_1 \neq y)(1-\mathrm{I}(f_2 \neq y))=\mathrm{E}_{f_1,f_2} 2(\mathrm{I}(f_1 \neq y)-\mathrm{I}(f_1 \neq y)\mathrm{I}(f_2 \neq y))$ and $R_S(G_Q)\!=\!\frac{1}{m}\sum_i\mathrm{E}_{f_1} \mathrm{I}(f_1 \neq y_i)$ (Eq.~(\ref{eqn:RSGQ})). Therefore $R_S(G_Q)\!=\!\frac{1}{m}\sum_i\mathrm{E}_{f_1} \mathrm{I}(f_1^i \neq y_i)\!=\!\frac{1}{m} \sum_i \mathrm{E}_{f_1,f_2}\mathrm{I}(f_1^i \neq y_i)\mathrm{I}(f_2^i \neq y_i)+\frac{1}{2m}\sum_i \mathrm{E}_{f_1,f_2}\mathrm{I}(f_1^i\neq f_2^i)=e_S(G_Q)+\frac{1}{2}d_S(G_Q)$, ($f_1^i=\!f_1(\mathbf{x}_i)$). Substituting $R_S(G_Q)=e_S(G_Q)\\+\frac{1}{2}d_S(G_Q)$ into Theorem~\ref{th:bound_sup}, then we obtain Theorem~\ref{th:bound_semi}.
\end{proof}

Since $d_S(G_Q)$ is independent of labels, it allows classifiers using the above bound to exploit unlabel data. Minimizing this risk $d_S(G_Q)$ would contract the posteriors over the stochastic classifier and the stochastic feature space, making classification and feature mapping less uncertain.

\subsection{Objective function and specification} \label{ssec:objfun}

Let $B(Q,C)= \frac{1}{1-e^{-C}} [1 - \exp \{ -J(Q)+\frac{1}{m}\ln \delta \}]$ be the upper bound in Theorem~\ref{th:bound_semi}, where $J(Q) = C (e_S(G_Q) + \frac{1}{2}d_S(G_Q)) + \frac{1}{m} \mathrm{KL}(Q \!\parallel\! P)$. Training a classifier with minimum risk means minimizing the upper bound $B(Q,C)$ w.r.t. $Q$ and $C$. Note that minimizing $B(Q,C)$ w.r.t. $Q$ equals to minimizing $J(Q)$ w.r.t. $Q$. Since unlabeled data are only available in estimating $d_S(G_Q)$, $J(Q)$ over labeled data $S_l$ of $m_l$ examples and unlabeled data $S_u$ of $m_u$ examples can be written as
\[
J(Q)=C \left[ e_{S_l}(G_Q) + \frac{1}{2}d_{S_u}(G_Q) \right] + \frac{1}{m}\mathrm{KL}(Q \!\parallel\! P)
\]
where $m=m_l+m_u$. This form enables us to derive the analytical form of posterior distributions $Q$ of the classifier and the hidden variables. Apply the above bound to the stochastic classifier defined in Eq.~(\ref{eqn:classifer}) and set $\mathbf{v}=(\mathbf{h}_+,\mathbf{h}_-)$, we have
\begin{equation}
f = f_{\mathbf{w}}, \quad \mathrm{E}_{f\sim Q} [\cdot] = \mathrm{E}_{\mathbf{w}\sim Q}[\cdot]
\end{equation}
Then learning PAC-Bayes bound classifier with the generative model embedding is to minimize the objective function $J(Q)$ w.r.t. the posterior $Q(\mathbf{w},\mathbf{h}_+,\mathbf{h}_-)$ and parameters $\theta_+,\theta_-$:
\begin{eqnarray}
\!\!\!\!J(Q) &\!\!\!\!=\!\!\!\!& C \left[ e_{S_l}(G_Q) + \frac{1}{2}d_{S_u}(G_Q) \right] + \frac{1}{m}[\mathrm{KL}(Q(\mathbf{w}) \parallel P(\mathbf{w}) ) \label{eqn:objfun} \\
&\!\!\!\!+\!\!\!\!& \mathrm{KL}(Q(\mathbf{h}_+) \parallel P(\mathbf{x},\mathbf{h}_+\,|\,\theta_+)) + \mathrm{KL}(Q(\mathbf{h}_-) \parallel P(\mathbf{x},\mathbf{h}_-\,|\,\theta_-))] \nonumber
\end{eqnarray}
where $G_Q$ is the linear stochastic classifier defined in Eq.~(\ref{eqn:classifer}); $P(\mathbf{x},\mathbf{h}_+\,|\,\theta_+)$ and $P(\mathbf{x},\mathbf{h}_-\,|\,\theta_-)$ are generative models for positive and negative classes respectively; the first row is the objective function for regular Gibbs classifier and the second row is the objective function for two generative models.

To compute the objective function $J(Q)$ in Eq.~(\ref{eqn:objfun}), we will need to have approximations or expressions for $e_{S_{l}}$, $d_{S_{u}}$ and $\mathrm{KL}(Q||P)$ that are computationally tractable. To derive these expressions, as were done in~\cite{langford2006tutorial}, we {\it assume} that the prior of the weight is Gaussian $P(\mathbf{w})= N(\mathbf{u}_0,\mathrm{I})$ and its posterior is also Gaussian except with a different mean, i.e., $Q(\mathbf{w})= N(\mathbf{u},\mathrm{I})$. Based on this assumption, we have:
\begin{equation} \label{eqn:specif1}
\mathrm{KL}(Q(\mathbf{w})\!\parallel\! P(\mathbf{w}))=\frac{1}{2} \parallel\! \mathbf{u}-\mathbf{u}_0 \!\parallel^2
\end{equation}
Using the assumption and Gaussian integrals~\cite{langford2006tutorial}, we have,
\begin{equation} \label{eqn:specif2}
\mathop{\mathrm{E}}\limits_{\mathbf{w}\sim Q}\mathbf{I}(f_{\mathbf{w}}(\mathbf{x},\mathbf{h}_+,\mathbf{h}_-)\neq y)=\Phi\left( {y \mathbf{u}\cdot \bar{\phi}}(\mathbf{x},\mathbf{h}_+,\mathbf{h}_-)  \right)
\end{equation}
where $\Phi(a)\!=\!\int_a^{\infty}\!\! \frac{1}{\sqrt{2\pi}}\exp{(-\frac{x^2}{2})}dx \!=\!\frac{1}{2}\mathrm{erfc}(\frac{a}{\sqrt{2}})$ and $\bar{\phi} \!=\! \frac{\phi(\mathbf{x},\mathbf{h}_+,\mathbf{h}_- )}{\!\parallel\! \phi(\mathbf{x},\mathbf{h}_+,\mathbf{h}_-) \!\parallel}$. Further, considering Eq.~(\ref{eqn:specif2}), we have the integration:
\begin{eqnarray}
\mathop{\mathrm{E}}\limits_{\mathbf{w}_1,\mathbf{w}_2 \sim Q} \!\!\! \mathrm{I}(f_{\mathbf{w}_1} \neq f_{\mathbf{w}_2}) &\!\!\!=\!\!\!& \mathop{\mathrm{E}}\limits_{\mathbf{w}_1,\mathbf{w}_2\sim Q} \!\!\! 2 \mathrm{I}(f_{\mathbf{w}_1} \neq 1)\mathrm{I}(f_{\mathbf{w}_2} \neq -1)  \label{eqn:specif3} \\
&\!\!\!=\!\!\!& 2 \Phi \!\left( {\mathbf{u}\cdot \bar{\phi}(\mathbf{x},\mathbf{h}_+,\mathbf{h}_-)}  \right) \Phi \!\left(\! {-\mathbf{u}\cdot \bar{\phi}(\mathbf{x},\mathbf{h}_+,\mathbf{h}_-)}  \right) \nonumber
\end{eqnarray}
With these formulas, we proceed to obtain an expression for $J(Q)$, and find $Q(\mathbf{h}_+,\mathbf{h}_-)$, $Q(\mathbf{w})$, $\theta_+$ and $\theta_-$ by minimizing $J(Q)$ with an iterative optimization procedure in the next section.

\section{Inference and parameter estimation}

In this section, we derive the learning procedure (inference and parameter estimation) of the proposed approach. Consider Eq.~(\ref{eqn:objfun}) and $e_S(G_Q)\!=\!\frac{1}{m} \sum\nolimits_i \mathrm{E}_{\mathbf{w}_1} \mathrm{I}(f_{\mathbf{w}_1}(\mathbf{x}_i) \!\neq\! y_i) \, \mathrm{E}_{\mathbf{w}_2} \mathrm{I}(f_{\mathbf{w}_2}(\mathbf{x}_i) \!\neq\! y_i)$ and $d_S(G_Q)\!=\!\frac{1}{m}\sum\nolimits_i \mathrm{E}_{\mathbf{w}_1,\mathbf{w}_2}\mathrm{I}(f_{\mathbf{w}_1}(\mathbf{x}_i) \neq f_{\mathbf{w}_2}(\mathbf{x}_i))$ (Th.\ref{th:bound_semi}), $J(Q)$ (average cost) over the training set $S\!=\!S_l \cup S_u$ is:
\begin{eqnarray*}
J(Q) \!=\! C \Bigg[ \frac{1}{m_l} \!\sum_{i=1}^{m_l}\! \mathrm{E}_Q \mathrm{I}(f_{\mathbf{w}}(\mathbf{x}_i)\neq y_i)+\frac{1}{m_u} \!\sum_{i=1}^{m_u}\! \mathrm{E}_Q \mathrm{I}(f_{\mathbf{w}_1}(\mathbf{x}_i)\neq f_{\mathbf{w}_2}(\mathbf{x}_i)) \Bigg]   \\
+ \frac{1}{m^2} \sum_{i=1}^m  \mathrm{KL}(Q_i(\mathbf{h}_+,\mathbf{h}_-) \!\parallel\! P(\mathbf{x}_i,\mathbf{h}_+,\mathbf{h}_-) )  + \frac{1}{m}\mathrm{KL}(Q(\mathbf{w}) \!\parallel\! P(\mathbf{w}))
\end{eqnarray*}
where the terms $\mathrm{E}_Q \mathrm{I}(f_{\mathbf{w}}(\mathbf{x}_i)\neq y_i)$, $\mathrm{E}_Q \mathrm{I}(f_{\mathbf{w}_1}(\mathbf{x}_i)\neq f_{\mathbf{w}_2}(\mathbf{x}_i))$ and $\mathrm{KL}(Q(\mathbf{w}) \!\parallel\! P(\mathbf{w}))$ have been respectively given by Eq.~(\ref{eqn:specif2}), Eq.~(\ref{eqn:specif3}) and Eq.~(\ref{eqn:specif1}); $P(\mathbf{x}_i,\mathbf{h}_+,\mathbf{h}_-)$ is the abbreviation of $P(\mathbf{x}_i,\mathbf{h}_+)P(\mathbf{x}_i,\mathbf{h}_-)$. We now show how an EM-like iterative procedure~\cite{jordan1999variational} can be used to learn the stochastic feature space and the Gibbs classifier simultaneously.

\subsection{Inference: minimize $J(Q)$ w.r.t. $Q_i(\mathbf{h}_+,\mathbf{h}_-)$}

In the {\it first step}, we fix $Q(\mathbf{w})$, $\theta_+$, $\theta_-$, and minimize $J(Q)$ w.r.t. $Q_i(\mathbf{h}_+,\mathbf{h}_-)$, subject to $\int\! Q_i(\mathbf{h}_+,\mathbf{h}_-) {d}\mathbf{h}_+ {d}\mathbf{h}_-\!=\!1$. Benefiting from the explicit bound (Th.~2), we has the following solution:
\begin{equation} \label{eqn:post_fact}
Q_i(\mathbf{h}_+,\mathbf{h}_-) = \frac{1}{Z_i} P(\mathbf{h}_+,\mathbf{h}_-,\mathbf{x}_i) \exp \left\{-C \mathrm{E}_{Q(\mathbf{w})} [\varphi_i] \right\} 
\end{equation}
where $\varphi_i = \frac{m^2}{m_l}\mathrm{I}(f_{\mathbf{w}} \neq y_i)$ if $\mathbf{x}_i\in S_l$ and $\varphi_i = \frac{m^2}{2m_u}\mathrm{I}(f_{\mathbf{w}_1}\neq f_{\mathbf{w}_2})$ if $\mathbf{x}_i\in S_u$. Note that $\mathrm{E}_{Q(\mathbf{w})} [\varphi_i]$ is given by Eq.~(\ref{eqn:specif2}) and Eq.~(\ref{eqn:specif3}).
The fact that the output of classifier is inside the expression for posteriors means that the generative models are being tuned as well when the classifier is being optimized during the minimization of PAC-Bayes bound. This tuning inhibits those examples of $\mathbf{h}_+$, $\mathbf{h}_-$ that lead to misclassification and encourages those with less misclassification. Sampling from this posterior is simple using {\it Gibbs-rejection sampling}, because $P(\mathbf{h},\mathbf{x}_i)$ can be directly used as the comparison function since $P(\mathbf{h},\mathbf{x}_i)\exp(\cdot) \leq P(\mathbf{h},\mathbf{x}_i)$ ($\mathrm{E}_{Q(\mathbf{w})} [\varphi_i]\geq 0$ as $\mathrm{I}(\cdot)$ is a zero-one output function). Considering the $j$-th example $\mathbf{h}_{ij}$ drawn from $P(\mathbf{h},\mathbf{x}_i)$, we reject it if $Q_i(\mathbf{h}_{ij}) \!<\! r_j$ where $r_j$ is drawn from the uniform distribution over $[0,P(\mathbf{h}_{ij},\mathbf{x}_i)]$.

\subsection{Parameter estimation: minimize $J(Q)$ w.r.t parameters}

In the {\it second step}, we fix the posteriors $Q_i(\mathbf{h}_+,\mathbf{h}_-)$ and parameters $\theta_+,\theta_-$, and determine the posterior distribution $Q(\mathbf{w})$. Instead of sampling from $Q(\mathbf{w})$, we directly determine its parameter $\mathbf{u}$ by minimizing $J(Q)$ w.r.t. $\mathbf{u}$. Since $J(Q)$ w.r.t. $\mathbf{u}$ is intractable, we resort to minimizing its upper bound $J(\mathbf{u})$ (see the Appendix) w.r.t. $\mathbf{u}$. The gradient of $J(\mathbf{u})$ w.r.t. $\mathbf{u}$ is:
\begin{eqnarray}
\frac{\partial J(\mathbf{u})}{\partial \mathbf{u}}  &\!\!\!=\!\!\!& \frac{1}{m}(\mathbf{u} \!-\! \mathbf{u}_0) - \frac{C}{m_l n} \!\sum\nolimits_{i,j=1}^{m_l,n}\!  g (y_i\mathbf{u}\cdot \bar{\phi}_{ij}) \,y_i\, \bar{\phi}_{ij}  \\
&\!\!\!+\!\!\!& \frac{C}{m_u n} \!\sum\nolimits_{i,j=1}^{m_u,n}\! g(\mathbf{u}\cdot \bar{\phi}_{ij}) \left[ \Phi(\mathbf{u} \cdot \bar{\phi}_{ij}) - \Phi(-\mathbf{u}  \cdot \bar{\phi}_{ij}) \right] \bar{\phi}_{ij} \nonumber
\end{eqnarray}
where $g(\cdot)$ is the gaussian function with mean $\mu=0$ and std $\delta=1$. The gradient of $B(Q,C)$ with respect to $C$ is:
\begin{eqnarray}
\frac{\partial B(Q,C)}{\partial C} &\!\!\!\!\!=\!\!\!\!\!& \frac{-e^{-C}}{(1-e^{-C})^2} \left[1-\exp{\Big(-J(Q) \!-\! \frac{1}{m}\log \delta \Big)}\right] \\
&\!\!\!\!\!+\!\!\!\!\!& \frac{1}{1-e^{-C}}[R_{S_l}(G_Q)+d_{S_u}(G_Q)] \exp{ \Big(\!-\!J(Q) \!-\! \frac{1}{m}\log \delta \Big)} \nonumber
\end{eqnarray}

In the {\em third step}, we fix $Q_i(\mathbf{h}_+,\mathbf{h}_-)$, $\mathbf{u}$ and update parameters $\theta_+,\theta_-$. Note only the third term of Eq.~(\ref{eqn:objfun}), i.e., the objection function of the positive model, involves $\theta_+$. {\em So the update rules of $\theta_+$, derived by minimizing Eq.~(\ref{eqn:objfun}) w.r.t. $\theta_+$, are same as those of the original generative model}. Similarly for $\theta_-$.

The {\it learning} procedure is summarized in Algorithm~\ref{alg:mixture}. In {\em classification}, similar with~\cite{germain2009pac}, we use the decision rule of majority vote $\hat{y}=\mathrm{sign}[\frac{1}{n}\sum_{j=1}^n \mathrm{E}_{P(\mathbf{h}_+,\mathbf{h}_-|\mathbf{x}_i)Q(\mathbf{w})}\mathbf{w}\cdot \phi(\mathbf{x}_i,\mathbf{h}_+,\mathbf{h}_-)]\simeq \mathrm{sign}[\frac{1}{n}\sum_{j=1}^n \mathbf{u}\cdot \phi(\mathbf{x}_i,\mathbf{h}_{+ij},\mathbf{h}_{-ij})]$ with $n=5$ and $(\mathbf{h}_{+ij},\mathbf{h}_{-ij})$ being the $j$-th example drawn from $P(\mathbf{h}_+,\mathbf{h}_-|\mathbf{x}_i)$.
\begin{algorithm}
  \caption{Inference and learning}
  \label{alg:mixture}
  \begin{algorithmic}[1]
    \STATE \textbf{input:} data set $S_l,S_u$, and $S'_l,S'_u$ are fractions of $S_l,S_u$
    \STATE initialize $\hat{\mathbf{u}},\hat{\theta}_{+},\hat{\theta}_{-}$, learning rates $\gamma_u,\gamma_c$, and $\delta=0.05$
    \STATE $\hat{\mathbf{u}}_0 \leftarrow \min_{\mathbf{u}}R_{S'_l}(G_Q)+\frac{1}{2}d_{S'_u}(G_Q)$
    \REPEAT
      \FOR{$i=1$ to $m$}
        \STATE sample $Q_i(\mathbf{h}_+,\mathbf{h}_-)$ using Gibbs-rejection sampling
      \ENDFOR
      \STATE update $\hat{\theta}_{+}$ with $\{\mathbf{h}_{+ij}\}_{ij}$ ($\mathbf{x}_i\in S_l^+$) using the rules of the original generative model. Similar for $\hat{\theta}_{-}$.  
      \STATE $\hat{\mathbf{u}} \leftarrow \hat{\mathbf{u}} - \gamma_u \frac{\partial J(\mathbf{u})}{\partial \mathbf{u}}$, $\;$ $C \leftarrow C - \gamma_c \frac{\partial B}{\partial C}$
    \UNTIL{convergence}
    \STATE \textbf{output:} $\hat{\mathbf{u}}_0$, $\hat{\mathbf{u}}$, $\hat{\theta}_{+}$, $\hat{\theta}_{-}$
  \end{algorithmic}
\end{algorithm}

\section{Experiments}
\label{sec:experiment}

This section empirically evaluates the proposed method stochastic feature mapping (SFM) and related methods on general classification tasks, scene recognition and protein sequence classification respectively. For multiple-class classification problem, we divide it into binary classification problems, each of which is an {\it one-versus-rest problem} that distinguishes one class from others. For each binary problem, we randomly partition the positive examples into $50\%$ training and $50\%$ test sets, and similarly for negative examples. We test each binary classification problem on $20$ random partitions, and report the average results. For the semi-supervised version, we use $25\%$ of test examples as unlabeled data. Two {\em related and general methods}, Fisher score (FS)~\cite{jaakkola1999exploiting} and free energy score space (FESS)~\cite{perina2009fess}, and some other state-of-the-art methods are also tested for comparison.

There are two points in implementation. First, the optimization procedures of $\mathbf{u}_0$ and $\mathbf{u}$ may suffer from the local minima problem, resulting in poor solution. The strategy adopted by~\cite{germain2009pac} is to perform the optimization for $10 \!\sim\! 100$ trials where a new random initial point within the range $[-20,20]^d$ is used in each trial. Second, the value of parameter $C$ has been shown to be important. Another effective strategy experimented is to assess the performance using 10-fold cross-validation.

\subsection{Deriving a general classification tool}
\label{ssec:uci}

In the first experiment, we derive a general classification method by applying the proposed framework to a simple yet general generative model, Gaussian mixture model. Let $\mathbf{x}\in \mathbb{R}^d$ be the observed variable; $\mathbf{z}=\{z_1,\cdots,z_k\}$ be the hidden binary indicate vector for $K$ mixture components, and assume the covariance matrix be diagonal; $\mathbf{a}=\{a_1,\cdots,a_k\}$ be the parameters of the approximate posterior of $\mathbf{z}$. The elements of the feature mapping $\phi$ of this model are $\{\mathbf{z}_i (\mathbf{x}^T,\mathrm{diag}(\mathbf{x}\mathbf{x}^T),1),z_i\log a_i\}_{i=1}^K$. The posterior of $\mathbf{z}$ can be easily derived from Eq.~(\ref{eqn:post_fact}). The number of mixture components is configured to $K=4$ throughout the experiment.

We select $8$ data sets from UCI database for evaluation, preferring those with no missing entities. The number of classes of each data set is between $2$ and $15$. The number of examples of each class varies from $14$ and $673$. The dimensionality is between $9$ and $90$. We compare our method SFM with Adaboost~\cite{freund1995boost}, SVM~\cite{vapnik2000nature}, localized multiple kernel learning (LMKL)~\cite{gonen2008lmkl} and PAC-Bayes gradient descent PBGD3~\cite{germain2009pac}. The average results are reported in Table~\ref{tab:uci}. It shows that SFM is adaptive to different data sets and outperforms other methods in half of the data sets. It is also worth noting that the linear version of PBGD3 does work well in these evaluation. The results of semi-supervised version is presented in Fig.~\ref{fig:gk_perf_centnum}.

\begin{table*}[t]
\caption{Classification accuracy ($\%\pm$std) on UCI database ({\it one-versus-rest} on each dataset).}
\label{tab:uci}
\begin{center}
\begin{tabular}{|c|c|c|c|c|c|c|}
\hline
DATA     & Adaboost  & SVM   & LMKL   & PBGD3       & SFM-GMM \\
\hline\hline
Cancer   & $93.24\pm1.26$           & $\mathrm{\textbf{96.80}}\pm\mathrm{\textbf{1.79}}$      & $96.41\pm0.97$   & $93.98\pm1.52$    & $95.02\pm0.95$  \\
\hline
Tissue   & $88.55\pm5.91$           & $78.17\pm12.27$     & $87.69\pm5.24$   & $88.14\pm0.50$    & $\mathrm{\textbf{89.04}}\pm\mathrm{\textbf{3.12}}$   \\
\hline
Wine     & $92.98\pm3.42$           & $\mathrm{\textbf{97.73}}\pm\mathrm{\textbf{1.86}}$     & $95.48\pm4.10$   & $92.22\pm12.63$    & $95.79\pm1.33$  \\
\hline
Sonar    & $70.87\pm4.76$           & $73.11\pm3.25$     & $80.21\pm1.52$   & $75.52\pm5.70$    & $\mathrm{\textbf{80.60}}\pm\mathrm{\textbf{4.24}}$  \\
\hline
Credit   & $84.74\pm1.30$           & $84.74\pm1.79$     & $81.92\pm1.41$   & $83.53\pm1.82$    & $\mathrm{\textbf{84.89}}\pm\mathrm{\textbf{1.30}}$  \\
\hline
SpHeart  & $78.65\pm2.08$           & $74.66\pm3.56$     & $80.38\pm3.40$   & $79.70\pm0.65$    & $\mathrm{\textbf{80.84}}\pm\mathrm{\textbf{0.49}}$  \\
\hline
Libras   & $92.97\pm1.76$           & $87.54\pm7.01$     & $\mathrm{\textbf{96.58}}\pm\mathrm{\textbf{1.78}}$   & $94.52\pm2.80$    & $95.61\pm3.34$  \\
\hline
Steel    & $89.47\pm9.08$           & $86.43\pm9.16$     & $\mathrm{\textbf{92.63}}\pm\mathrm{\textbf{8.14}}$   & $87.30\pm8.26$    & $89.10\pm8.49$  \\
\hline
\end{tabular}
\end{center}
\end{table*}

\begin{figure}[t]
\centering
\includegraphics[width=1.12in]{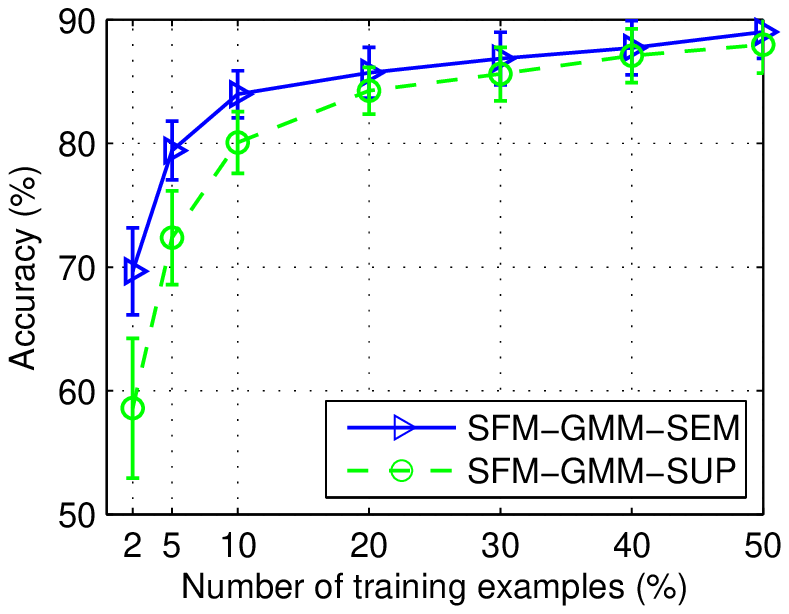}
\includegraphics[width=1.12in]{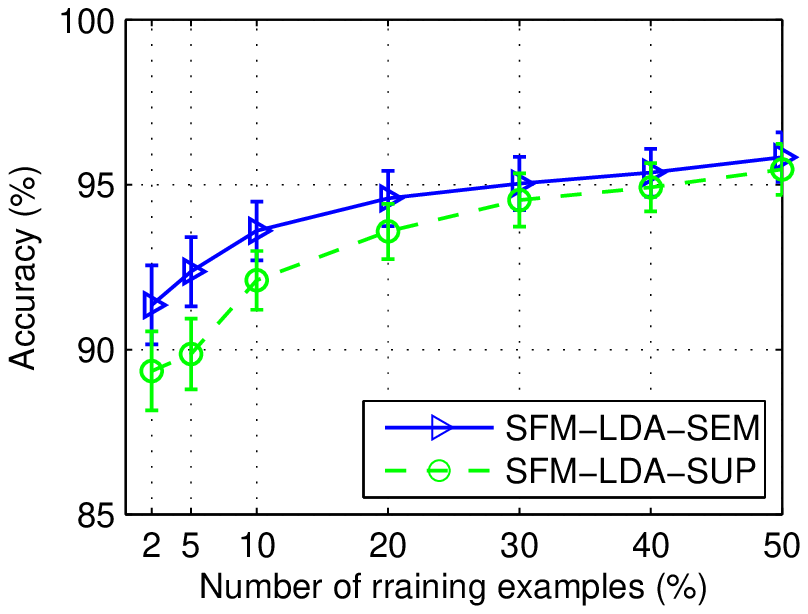}
\includegraphics[width=1.12in]{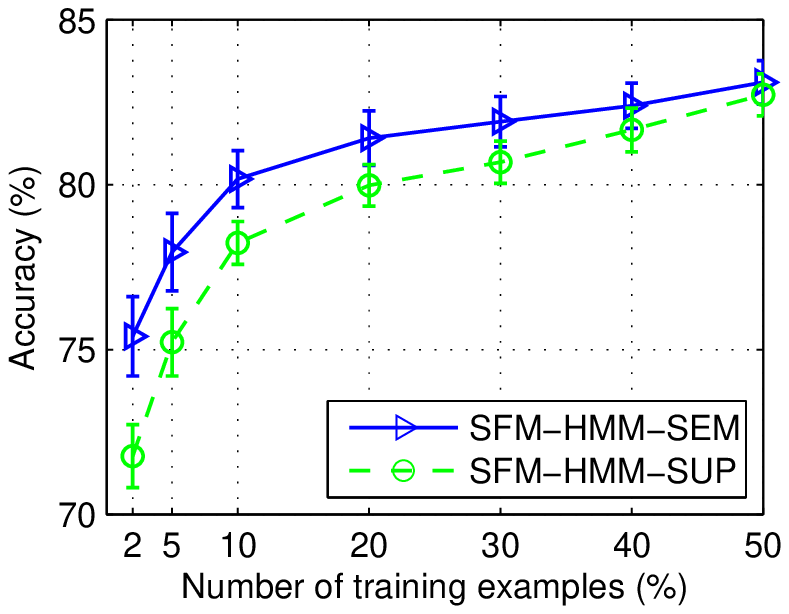}
\caption{Classification accuracy ($\%$) for varying number of training examples. {\it left}: Tissue (UCI); {\it middle}: Highway (scene); {\it right}: Superfamily \#2 (protein).}
\label{fig:gk_perf_centnum}
\end{figure}

\subsection{Scene recognition}

\begin{table*}[t]
\caption{Accuracy ($\%\pm$std.) of {\it one-versus-rest} scene recognition.}
\label{tab:scene}
\begin{center}
\begin{tabular}{|c|c|c|c|c|c|}
\hline
SCENE  & PHOW-SVM   & LDA-MAP   & FS-LDA   & FESS-LDA    & SFM-LDA \\
\hline\hline
Coast      & $90.66\pm0.65$    & $83.85\pm0.92$         & $90.42\pm0.34$          & $93.89\pm0.46$   & $\mathrm{\textbf{94.06}}\pm\mathrm{\textbf{0.64}}$  \\
\hline
Forest     & $96.49\pm0.39$    & $94.94\pm0.46$         & $94.45\pm0.46$          & $97.92\pm0.26$   & $\mathrm{\textbf{98.10}}\pm\mathrm{\textbf{0.32}}$  \\
\hline
Mountain   & $92.58\pm0.64$    & $84.99\pm1.78$         & $88.62\pm0.50$          & $93.29\pm0.47$   & $\mathrm{\textbf{93.80}}\pm\mathrm{\textbf{0.44}}$  \\
\hline
Country    & $\mathrm{\textbf{91.38}}\pm\mathrm{\textbf{0.71}}$      & $72.30\pm1.74$         & $87.40\pm0.46$          & $90.62\pm0.33$   & $90.82\pm0.62$   \\
\hline
Highway    & $95.27\pm0.49$    & $81.50\pm1.28$         & $92.48\pm0.22$          & $94.67\pm0.34$   & $\mathrm{\textbf{95.46}}\pm\mathrm{\textbf{0.30}}$  \\
\hline
InsideCity & $93.96\pm0.62$    & $85.14\pm1.74$         & $90.79\pm0.14$          & $94.26\pm0.65$   & $\mathrm{\textbf{95.27}}\pm\mathrm{\textbf{0.35}}$  \\
\hline
Street     & $93.89\pm0.64$    & $76.46\pm1.23$         & $93.76\pm0.24$          & $94.21\pm0.42$   & $\mathrm{\textbf{95.03}}\pm\mathrm{\textbf{0.47}}$  \\
\hline
Building   & $94.40\pm0.49$    & $87.85\pm0.55$         & $92.83\pm0.57$          & $\mathrm{\textbf{96.06}}\pm\mathrm{\textbf{0.51}}$  &  $95.81\pm0.38$  \\
\hline
\end{tabular}
\end{center}
\end{table*}

We evaluate our SFM method and compare its performance against comparable methods on a typical vision task, scene recognition. In this task, visual words are used for image representation for its robustness to topic and spatial variance. We use latent Dirichlet allocation (LDA)~\cite{blei2003latent} to model the distributions of visual words, and derive a recognition tool under the proposed framework. Like~\cite{griffiths2004gibbslda}, we sample the topic variable using collapsed Gibbs sampling and reject examples according to the rule for Eq.~(\ref{eqn:post_fact}). We fix the parameter $\alpha$ and allow $\beta$~\cite{griffiths2004gibbslda} to be updated. Let $w,z$ respectively indicate word and topic, and $\gamma$ be the parameter of the approximate posterior of $z$. The elements of the feature mapping $\phi$ of a model are $\{z_{nk},w_n z_{nk},z_{nk}\log \gamma_{nk} \}_{n,i,k}$ where  $n,i,k$ index word, term and topic respectively. For FS~\cite{jaakkola1999exploiting} and FESS~\cite{perina2009fess}, we extract features from the trained LDA model and deliver to SVM. The number of topics of LDA is set to $50$.

The CVCL scene dataset is chosen for evaluation. It contains $4$ artificial scenes and $4$ natural scenes. For each image, dense SIFT descriptors~\cite{lowe2004distinctive} are extracted from $20\!\times\!20$ grid patches over $4$ scales. These descriptors are quantized to visual words using a code book (50 centers) clustered from some random selected descriptors. The resulting visual words of an image are in the form of histogram where each bin corresponds to a code center of the code book. The evaluation results are summarized in Table~\ref{tab:scene}. Our results compare well with PHOW~\cite{vedaldi2009multiple} which is a state-of-the-art feature for scene recognition. The results of semi-supervised learning are shown in Fig.~\ref{fig:gk_perf_centnum}, demonstrating unlabeled examples can help classification particularly when there are few labeled examples.

\subsection{Protein classification}
\label{ssec:protein}

\begin{table*}[t]
\caption{Accuracy ($\%\pm$std.) of {\it one-versus-rest} protein recognition.}
\label{tab:protein}
\begin{center}
\begin{tabular}{|c|c|c|c|c|c|}
\hline
 SUP.FAM.   & 2GRAM-SVM  & HMM-MAP  & FS-HMM   & FESS-HMM   & SFM-HMM \\
\hline\hline
\# 1    & $78.79\pm1.13$   & $80.91\pm1.53$    & $80.03\pm0.78$            & $80.12\pm0.84$   & $\mathrm{\textbf{82.75}}\pm\mathrm{\textbf{0.96}}$  \\
\hline
\# 2    & $79.01\pm0.97$   & $80.10\pm0.51$    & $77.56\pm0.64$            & $78.96\pm0.59$   & $\mathrm{\textbf{83.08}}\pm\mathrm{\textbf{0.55}}$  \\
\hline
\# 3    & $75.19\pm0.86$   & $77.92\pm0.79$    & $73.31\pm0.21$            & $73.35\pm0.41$   & $\mathrm{\textbf{79.28}}\pm\mathrm{\textbf{0.64}}$  \\
\hline
\# 4    & $96.01\pm0.33$   & $95.10\pm0.39$    & $94.27\pm0.37$            & $\mathrm{\textbf{97.58}}\pm\mathrm{\textbf{0.13}}$  &  $96.72\pm0.63$   \\
\hline
\end{tabular}
\end{center}
\end{table*}

To evaluate the capability of the proposed approach in dealing with variable length sequences, we apply the proposed framework to remote homology recognition. The problem here assigns test protein sequences to the domain superfamilies defined in the SCOP (1.53) taxonomy tree according to functions of proteins. The protein sequence data is obtained from ASTRAL database with E-value threshold of $10^{-25}$ to reduce similar sequences. We uses four labeled domain superfamilies, metabolism, information, intra-cellular processes and extra-cellular processes for evaluation. The numbers of sequences are 804, 950, 695 and 992 respectively. Each protein sequence is a string composed of $22$ distinct letters, and the string length varies from $20$ to $994$.

Hidden Markov model (HMM)~\cite{rabiner1989tutorial} is used to model the distribution over protein sequences for its ability in handling sequences with variable length. The number of output states is $22$, and the number of hidden states is set to $10$. Let $\mathbf{x}$ be the sequence with length $T_{\mathbf{x}}$, where $\mathbf{x}^t$ be the binary indicator where $x_k^t=1$ if the $k$-th state of $K$ possible ones is selected at time $t$. Let $\mathbf{q}^t$ be the binary state indicator where $q_i^t=1$ if the $i$-th state of $M$ possible ones is selected at time $t$; $A_{M\times M}$ be the transition probabilities of the approximate posteriors. The elements of the feature mapping $\phi$ can be written as $\{q_i^0, \sum_{t=0}^{T_{\mathbf{x}}-1}\!\! q_i^t q_j^{t+1}, \sum_{t=0}^{T_{\mathbf{x}}-1}\!\! q_i^t q_j^{t+1} \!\log A_{ij}, \sum_{t=0}^{T_{\mathbf{x}}}\! q_i^t x_k^t \}_{i,j,k}$. With the hidden states of the input sequence inferred by Baum-Welch algorithm~\cite{baum1970maximization}, it is easy to estimate the posterior transition probabilities conditioned on $\mathbf{x}$. Using the sampling distribution derived in Eq.~(\ref{eqn:post_fact}), we are able to draw the examples of hidden states and re-estimate their posterior. The results are reported in Table~\ref{tab:protein}. The 2-gram feature is actually the transition probability of observed states of a sequence, i.e. $\{ \frac{1}{T_c}\sum_{t=0}^{T_c-1} x_i^t x_k^{t+1}\}_{i,k}$. The difference of the performance of the first four methods are not significant except on family \#3. The results of semi-supervised learning are reported in Fig.~\ref{fig:gk_perf_centnum}, which shows improvement on few training samples.

\section{Conclusions}
\label{sec:conclusion}

This paper presents a framework to incorporate the abilities of generative model and discriminative model for classification under the PAC-Bayes theory. The bridge of this incorporation is a stochastic feature mapping which is derived from the linear form of the practical MAP classifier and is independent with the parameters of the adopted generative models. Under this framework, the derived stochastic feature mapping and generative models can be tuned during the training of the classifier. A major difficulty is the non-convexity of the objective function, where local minima can hamper the solution. Our approach can benefit from the development or exploitation of more robust and efficient optimization methods.


\section*{Appendix}

Since $J(Q)$ is intractable, we derive its upper bound by fixing $\theta_+,\theta_-$, $Q_i(\mathbf{h}_+,\mathbf{h}_-)$. Using Eq.~(\ref{eqn:specif1}) and Eq.~(\ref{eqn:post_fact}), we have
\begin{eqnarray*}
\!\mathrm{KL}(Q \!\parallel\! P ) &\!\!\!\!\!=\!\!\!\!\!& \mathrm{KL}(Q(\mathbf{w},\mathbf{h}_+,\mathbf{h}_-) \!\parallel\! P(\mathbf{w})P(\mathbf{x},\mathbf{h}_+\,|\,\theta_+)P(\mathbf{x},\mathbf{h}_-\,|\,\theta_-)) \label{eqn:objfun2}  \\
&\!\!\!\!\!=\!\!\!\!\!& \mathop{\mathrm{E}}\limits_{Q_{\mathbf{w}}} \left[\log \frac{Q(\mathbf{w})}{P(\mathbf{w})}  \right] + \frac{1}{m} \sum_{i=1}^m \mathop{\mathrm{E}}\limits_{Q_i} \left[\log \frac{Q_i(\mathbf{h}_+)Q_i(\mathbf{h}_-)}{P(\mathbf{x}_i,\mathbf{h}_+)P(\mathbf{x}_i,\mathbf{h}_-)} \right]  \\
&\!\!\!\!\!=\!\!\!\!\!& \frac{1}{2} \!\!\parallel\! \mathbf{u} \!-\! \mathbf{u}_0 \!\parallel^2\!\! - Cm \!\left[ e_{S_l}(G_Q) \!+\! \frac{1}{2}d_{S_u}\!(G_Q) \right] \!-\! \frac{1}{m} \!\!\sum_{i=1}^m \! \log Z_i
\end{eqnarray*}
where $Q_{\mathbf{w}}=Q(\mathbf{w})$, $Q_i=Q_i(\mathbf{h}_+,\mathbf{h}_-)$ and
\begin{eqnarray*}
\!\sum_{i=1}^m \!\log Z_i &\!\!\!\!\!=\!\!\!\!\!& \sum_{i=1}^m \log \mathop{\mathrm{E}}_{Q_i}[\exp \left\{-C \mathop{\mathrm{E}}\nolimits_{Q(\mathbf{w})} [\varphi_i] \right\}] \geq \!\sum_{i=1}^m \mathop{\mathrm{E}}_{Q_i} [-C \mathop{\mathrm{E}}\nolimits_{Q_{\mathbf{w}}} [\varphi_i]] \\
&\!\!\!\!\!=\!\!\!\!\!& \sum_{i=1}^{m_l} \!\mathop{\mathrm{E}}\limits_{Q_{\mathbf{w}}Q_i}\! \left[\frac{C m^2}{m_l} \mathrm{I}(f_{\mathbf{w}}^i \!\neq\! y_i)\right] + \sum_{i=1}^{m_u} \!\mathop{\mathrm{E}}\limits_{Q_{\mathbf{w}}Q_i}\! \left[\! \frac{Cm^2}{2m_u} \mathrm{I}(f_{\mathbf{w}_1}^i \!\neq\! f_{\mathbf{w}_2}^i) \!\right]  \\
&\!\!\!\!\!\simeq\!\!\!\!\!&  \frac{Cm^2}{m_l n} \!\sum_{i,j=1}^{m_l,n} \!\Phi\! \left( y_i \mathbf{u} \!\cdot\! \bar{\phi}_{ij} \right)
+ \!\frac{Cm^2}{m_u n} \!\sum_{i,j=1}^{m_u,n} \!\Phi\! \left(\mathbf{u} \!\cdot\! \bar{\phi}_{ij} \right) \!\Phi\! \left(\!-\mathbf{u} \!\cdot\! \bar{\phi}_{ij} \right)
\end{eqnarray*}
where the inequality is derived by applying Jensen's inequality; $f_{\mathbf{w}}^i\!=\!f_{\mathbf{w}}(\mathbf{x}_i)$; $ \bar{\phi}_{ij} \!=\! \frac{\phi(\mathbf{x}_i,\mathbf{h}_{+ij},\mathbf{h}_{-ij})}{\parallel\! \phi(\mathbf{x}_i,\mathbf{h}_{+ij},\mathbf{h}_{-ij}) \!\parallel}$ where $(\mathbf{h}_{+ij},\mathbf{h}_{-ij})$ represents the $j$-th example drawn from $Q_i(\mathbf{h}_+,\mathbf{h}_-)$. Now we have all the pieces for $e_{S_{l}}$, $d_{S_{u}}$ and $\mathrm{KL}(Q\!\parallel\!P)$, and can obtain,
\begin{eqnarray*}
J(Q) &\!\!\!\!=\!\!\!\!& C \left[ e_{S_l}(G_Q) + \frac{1}{2}d_{S_u}(G_Q) \right] + \frac{1}{m}\mathrm{KL}(Q \parallel P )  \\
&\!\!\!\!=\!\!\!\!& \frac{1}{2m} \!\parallel\! \mathbf{u} - \mathbf{u}_0 \!\parallel^2  - \frac{1}{m^2} \sum_i \log Z_i \\
&\!\!\!\!\leq\!\!\!\!&  \frac{1}{2m} \!\parallel\! \mathbf{u}-\mathbf{u}_0 \!\parallel^2 + \frac{C}{m_l n} \sum_{i=1}^{m_l}\sum_{j=1}^n \Phi \left( y_i \mathbf{u}\cdot \bar{\phi}_{ij} \right) \\
&\!\!\!\!+\!\!\!\!& \frac{C}{m_u n}\sum_{i=1}^{m_u}\sum_{j=1}^n  \Phi \left(\mathbf{u}\cdot \bar{\phi}_{ij} \right)\Phi \left(-\mathbf{u}\cdot \bar{\phi}_{ij} \right) \triangleq J(\mathbf{u})
\end{eqnarray*}

\ifCLASSOPTIONcompsoc
  \section*{Acknowledgments}
\else
  \section*{Acknowledgment}
\fi

This work is supported by National Basic Research Program of China (Grant No. 2011CB302203), NSFC (Grant No. 60833009 and 60975012) and Microsoft Research Fellowship.

\ifCLASSOPTIONcaptionsoff
  \newpage
\fi



\bibliographystyle{IEEEtran}
\end{document}